%% file: main.tex
\documentclass[10pt,twocolumn,letterpaper]{article}
\usepackage[T1]{fontenc}
\usepackage[latin9]{inputenc}
\usepackage{multirow}
\usepackage{amsmath}
\usepackage{amsthm}
\usepackage{amssymb}
\usepackage{graphicx}

\makeatletter

\providecommand{\tabularnewline}{\\}

\theoremstyle{plain}
\newtheorem{prop}{\protect\propositionname}
\theoremstyle{plain}
\newtheorem{cor}{\protect\corollaryname}
\theoremstyle{remark}
\newtheorem{rem}{\protect\remarkname}

\usepackage[pagenumbers]{cvpr} 

\usepackage{booktabs}

\usepackage[pagebackref,breaklinks,colorlinks]{hyperref}

\usepackage[capitalize]{cleveref}
\crefname{section}{Sec.}{Secs.}
\Crefname{section}{Section}{Sections}
\Crefname{table}{Table}{Tables}
\crefname{table}{Tab.}{Tabs.}

\makeatother

\providecommand{\corollaryname}{Corollary}
\providecommand{\propositionname}{Proposition}
\providecommand{\remarkname}{Remark}

\begin{document}
\global\long\def\Real{\mathbb{R}}%
\global\long\def\Loss{\mathcal{L}}%
\global\long\def\W{\mathcal{W}}%
\global\long\def\Z{\mathcal{Z}}%
\global\long\def\s{\text{s}}%
\global\long\def\t{\text{t}}%
\global\long\def\id{\text{id}}%
\global\long\def\sty{\text{sty}}%
\global\long\def\Model{\text{AFS}}%

\title{Face Swapping as A Simple Arithmetic Operation}

\author{Truong Vu$^{1}$\\{\tt\small truong.vt183649@sis.hust.edu.vn} \and Kien Do$^{2}$\\{\tt\small k.do@deakin.edu.au} \and Khang Nguyen$^{1}$\\{\tt\small khang.nt183559@sis.hust.edu.vn} \and Khoat Than$^{1}$\\{\tt\small khoattq@soict.hust.edu.vn} \\
$^{1}$ Hanoi University of Science and Technology, Vietnam\\
$^{2}$ Applied Artificial Intelligence Institute (A2I2), Deakin University, Australia}
\maketitle

\begin{abstract}
\input{abstract.tex}
\end{abstract}

\section{Introduction\label{sec:Introduction}}

\input{intro.tex}

\section{Related Work\label{sec:Related-Work}}

\input{relate.tex}

\section{Method\label{sec:Method}}

\input{method.tex}

\section{Experiments\label{sec:Experiments}}

\input{exp.tex}

\section{Conclusion\label{sec:Conclusion}}

\input{discuss.tex}

\section{Ethical Considerations}

\input{ethic.tex}

\bibliographystyle{ieee_fullname}
\bibliography{latentgan}

\cleardoublepage{}

\appendix

\section{Appendix}

\input{appendix.tex}

\end{document}

%% file: abstract.tex
We propose a novel high-fidelity face swapping method called \emph{``Arithmetic
Face Swapping''} (AFS) that explicitly disentangles the intermediate
latent space $\W^{+}$ of a pretrained StyleGAN into the \emph{``identity''}
and \emph{``style''} subspaces so that a latent code in $\W^{+}$
is the sum of an ``identity'' code and a ``style'' code in the
corresponding subspaces. Via our disentanglement, face swapping (FS)
can be regarded as a simple arithmetic operation in $\W^{+}$, i.e.,
the summation of a source ``identity'' code and a target ``style''
code. This makes AFS more intuitive and elegant than other FS methods.
In addition, our method can generalize over the standard face swapping
to support other interesting operations, e.g., combining the identity
of one source with styles of multiple targets and vice versa. We implement
our identity-style disentanglement by learning a neural network that
maps a latent code to a ``style'' code. We provide a condition for
this network which theoretically guarantees identity preservation
of the source face even after a sequence of face swapping operations.
Extensive experiments demonstrate the advantage of our method over
state-of-the-art FS methods in producing high-quality swapped faces.
Our source code was made public at \url{https://github.com/truongvu2000nd/AFS}

%% file: intro.tex
Face swapping is the task of transferring the identity from a source
face image to a target face image while preserving the attributes
(styles) of the target face image such as pose, facial expression,
lighting, and background. It has attracted a lot of attentions from
community recently thanks to its practical applications in entertainment
\cite{kemelmacher2016transfiguring,chen2020simswap} and privacy protection
\cite{rossler2019faceforensics++}.

A large number of face swapping methods have been proposed. Early
methods can only handle source and target faces with similar poses,
or source faces with known identities \cite{bitouk2008face,korshunova2017fast}.
Later methods are able to deal with source and target faces having
very different styles, and can even work in unconstrained settings
\cite{bao2018towards,chen2020simswap,li2019faceshifter,nirkin2018face,nirkin2019fsgan}.
Among them, GAN-based methods \cite{chen2020simswap,li2019faceshifter}
often achieve better visual results than those using conventional
image blending models \cite{nirkin2018face,nirkin2019fsgan}.

Recent years have also witnessed the development of several powerful
GANs for image generation \cite{brock2018large,karras2020analyzing,karras2017progressive,karras2019style}.
StyleGAN \cite{karras2020analyzing,karras2019style} is probably the
current state-of-the-art, which been shown to be capable of generating
diverse high-fidelity images at megapixel resolution. Its superiority
mainly comes from the original architectural design of its generator.
In conventional GANs \cite{goodfellow2014generative}, an input noise
vector $z\in\Z$ directly passes through the generator $G$ in a feed-forward
manner. Meanwhile, in StyleGAN, $z$ is mapped to an intermediate
latent code $w\in\W$ which is then used to modulate the intermediate
feature maps of $G$ at different spatial resolutions. Many works
have shown that the latent space $\W^{+}$ of StyleGAN - an extension
of $\W$ - is more expressive and disentangled than $\W$ and $\Z$,
motivating the use of $\W^{+}$ for latent space manipulation \cite{abdal2019image2stylegan,abdal2020image2stylegan++}.

Inspired by the great capability of StyleGAN, several works have made
use of a pretrained StyleGAN to perform high-resolution face swapping.
Zhu et al. \cite{zhu2021one} proposed to learn a Face Transfer Module
(FTM) that combines the inverted latent codes of the source and target
faces to obtain a swapped latent code which will then be sent to the
StyleGAN generator to produce a swapped face. However, since the two
latent codes contain entangled information of identity and attributes,
directly fusing them together may cause the swapped result to retain
some information of the source attributes and the target identity.
Xu et al. \cite{xu2022high}, on the other hand, proposed to concatenate
the ``structure'' part of the source latent code with the disjoint
``appearance'' part of the target latent code to get an input latent
code for the StyleGAN generator. Unfortunately, this input latent
code still contains entangled information in its parts. Thus, the
generated face derived from it has mixed styles, and is considered
as a side output not the final result.

To address the above limitation, we propose a novel face swapping
method based on a pretraind StyleGAN that explicitly disentangles
the intermediate latent space $\W^{+}$ of a pretrained StyleGAN into
two \emph{subspaces} $\W_{\id}^{+}$ and $\W_{\sty}^{+}$ characterizing
identity and style respectively. Motivated by the expressiveness and
linear editability of $\W^{+}$ \cite{abdal2019image2stylegan,tov2021designing},
we hypothesize that we can represent the inverted latent code $w\in\W^{+}$
of a face image as the sum of an \emph{``identity'' code} $w_{\id}\in\W_{\id}^{+}$
and a \emph{``style'' code} $w_{\sty}\in\W_{\sty}^{+}$, i.e., $w=w_{\id}+w_{\sty}$.
This enables us to regard face swapping as a simple\emph{ arithmetic
operation} in $\W^{+}$, that is the summation of a \emph{source ``identity''
code} and a \emph{target ``style'' code}. Therefore, we name our
method \emph{``Arithmetic Face Swapping''} (AFS). Compared to existing
formulations of face swapping, ours is more intuitive and elegant.
In fact, the standard face swapping is just \emph{a special case }of
numerous operations that our identity-style disentanglement could
support. For example, we can take the (weighted) average of multiple
target style codes in $\W_{\sty}^{+}$ and combine the result with
a source identity code in $\W_{\id}^{+}$ via summation, which is
equivalent to face swapping with one identity and multiple styles.
Similarly, we can easily combine multiple identities together to create
a new identity that has not existed before without changing styles.
These kinds of face swapping, to the best of our knowledge, have never
been considered in previous works. We implement our identity-style
disentanglement by learning a style extractor network $h$ that maps
$w$ to $w_{\sty}$ and compute $w_{\id}$ as $w-h(w)$. We also provide
a condition for $h$ which theoretically guarantees that our method
can preserve the identity of the source image even when this image
goes through a sequence of face swapping operations with different
styles. This condition is then turned into an objective for training
$h$, along with other well-designed objectives to ensure that $w_{\id}$
and $w_{\sty}$ derived from $h$ capture the true identity and style
information in $w$ respectively. Extensive experiments demonstrate
that AFS outperforms existing state-of-the-art face swapping methods
in generating high-quality results, and preserving the source identity
and the target attributes.

%% file: relate.tex
\paragraph{Face Swapping}

Face swapping has long been an interesting and difficult problem in
computer vision \cite{blanz2004exchanging}. Early approaches are
mainly based on conventional image processing techniques, thus, require
the source and target faces to have similar poses and appearances
\cite{bitouk2008face,wang2008facial}. With the development deep neural
networks, especially deep generative models such as VAEs \cite{kingma2013auto}
and GANs \cite{goodfellow2014generative}, more advanced face swapping
methods have been proposed. Korshunova et al. \cite{korshunova2017fast}
regard face swapping as a style transfer problem in which content
is the target face, and style is the source identity. They leverage
a Texture Network \cite{ulyanov2016texture} to model the identity
of each subject in the training data. However, their method cannot
generalize to unseen subjects. Nirkin et al. \cite{nirkin2018face}
learn an occlusion-aware face segmentation model and combine it with
off-the-shelf 3D morphable face models (3DMM) \cite{paysan20093d,zhu2016face}
and image blending model \cite{perez2003poisson} to form a complete
pipeline for face swapping that can handle unconstrained face images.
Their subsequent work \cite{nirkin2019fsgan} further enhances this
pipeline by proposing a novel reenactment generator to better preserve
the target pose and expression, and a novel blending loss to better
preserve the target skin color. However, due to the limitation of
the Possion blender in simulating complex textures and lighting effects,
methods based on this technique \cite{nirkin2018face,nirkin2019fsgan,rossler2019faceforensics++}
tend to create unnatural swapped faces. Therefore, there has been
another line of works that make use of GANs to generate real looking
swapped faces. FSNet \cite{natsume2018fsnet} extracts the representation
of the source face region and the representation of the target face
landmarks, and sends them along with the target face-masked image
to an UNet-like generator to generate a swapped face. Since only the
target landmark information is given as input to the generator, this
model may not be able to preserve the expression and skin color of
the target face. IPGAN \cite{bao2018towards} generates a swapped
face from the concatenation of the source identity and the target
attribute vectors. This method can generalize to unseen identities
quite well but often fails to preserve details in the target face
such as wrinkles or beards. FaceShifter \cite{li2019faceshifter}
proposes a well-designed generator that combines the source identity
vector with the target attribute feature maps at various spatial resolutions.
This generator can generate good-looking a swapped face that faithfully
respects the style and resolution of the target face. The swapped
face then goes through a refinement network for facial occlusion correction.
SimSwap \cite{chen2020simswap} uses an ID Injection Module (IIM)
to integrate the source identity into the target features before sending
to the generator. This IIM contains multiple AdaIN layers \cite{huang2017arbitrary}
whose parameters depend on the source identity. HiFiFace \cite{wang2021hififace}
extends the source identity vector with the target 3D pose and expression
representation vectors extracted by a 3DMM model \cite{deng2019accurate},
and uses this extended source identity vector to modulate the generation
of a swapped face from the target face in nearly the same way as FaceShifter.
FaceInpainter \cite{li2021faceinpainter} is a two-stage face swapping
framework inspired by FaceShifter that also uses 3D pose and expression
as prior knowledge for face swapping. InfoSwap \cite{gao2021information}
leverages the Information Bottleneck principle \cite{alemi2016deep}
to achieve a good separation between identity and style information.

Besides the aforementioned methods, some methods make use of a pretrained
StyleGAN2 and are more related to ours. These methods usually map
the source and the target faces to the StyleGAN2 intermediate latent
codes, and use these latent codes to perform face swapping. MegaFS
\cite{zhu2021one} uses a Face Transfer Module (FTM) to fuse 14 high-level
components of the source and the target latent codes together. The
fused high-level components are then concatenated with the remaining
4 low-level components of the target latent codes to obtain a complete
swapped latent code which serves as the input to the StyleGAN2 generator
to generate a swapped face. HiRes \cite{xu2022high} uses both the
source and the target facial landmarks to modify the ``structure''
part of the source latent codes (consisting of the first 7 components),
and combines the modified one with the ``appearance'' part of the
target latent codes (consisting of the last 11 components) to obtain
a swapped latent code. This swapped latent code is then fed to the
StyleGAN2 generator not to generate a swapped face but to get the
intermediate feature maps. These intermediate feature maps will be
used to modulate another auto-encoder network that maps the target
face to the actual swapped face. 

\paragraph{GAN Inversion and GAN Latent Space Manipulation}

Due to large scopes of GAN Inversion and GAN Latent Space Manipulation,
in this paper, we only discuss works in these topics that are the
most related to ours. Given a well-trained GAN, we often want to find
the latent code that faithfully represents an input image. There are
three main approaches to the problem namely i) latent optimization
\cite{abdal2019image2stylegan,abdal2020image2stylegan++,creswell2018inverting,lipton2017precise},
ii) learning an encoder \cite{perarnau2016invertible,richardson2021encoding,tov2021designing},
and iii) the hybrid of learning an encoder and latent optimization
\cite{bau2019seeing,zhu2020domain}. The first approach directly optimizes
a random latent code by minimizing the reconstruction error between
the generated image w.r.t. the latent code and the input image. This
approach often yields accurate inversion but does not scale well since
we have to run multiple optimization steps for each input image. By
contrast, the second approach computes the latent code via a single
forward pass through the encoder, thus, is significantly faster at
inference time. The third approach learns an encoder to find good
initializations of latent codes, then optimizes these latent initializations.
This approach often enjoys the benefits of fast inference and good
accuracy from the first two approaches.

Many recent works have shown that semantically meaningful image editions
can be done via simple affine transformations in the latent space
of a pretrained GAN. Some works try to discover interpretable directions
in the GAN latent space in an unsupervised manner \cite{harkonen2020ganspace,shen2021closed,voynov2020unsupervised}.
However, they have to manually assign semantics to the discovered
directions. Supervised methods \cite{shen2020interfacegan,tewari2020pie,tewari2020stylerig,yang2021l2m},
by contrast, make use of pretrained attribute classifiers/regressors
to guide the search for semantic directions.

%% file: method.tex
Since our method makes use of a pretrained StyleGAN2 \cite{karras2020analyzing}
for face swapping, below we give a brief description of StyleGAN \cite{karras2019style}
and StyleGAN2 \cite{karras2020analyzing} before going into details
of our method.

\subsection{StyleGAN and StyleGAN2}

StyleGAN \cite{karras2019style} is a recent state-of-the-art variant
of GANs \cite{goodfellow2014generative} that is able to generate
high-fidelity images at high resolution. Its generator $G$ inherits
many advantageous designs from the style transfer literature \cite{huang2017arbitrary}.
Different from conventional GANs, in StyleGAN, the noise vector $z\in\Z$
is not fed directly to $G$ but is mapped to an \emph{intermediate}
latent code $w\in\W\subset\Real^{512}$ via an 8-layer MLP $f:\Z\rightarrow\W$.
This latent code $w$ controls the style and other details of a generated
image from coarse to fine levels by serving as input to 18 different
AdaIN layers \cite{huang2017arbitrary,dumoulin2016learned,ghiasi2017exploring}
of the generator at various spatial resolutions from 4$\times$4 to
1024$\times$1024 (2 AdaIN layers for each spatial resolution). Several
works \cite{shen2020interfacegan,harkonen2020ganspace} empirically
showed that $\W$ supports linear latent code manipulation as they
were able to find semantic directions in $\W$ corresponding to meaningful
disentangled attributes such as color change, zoom, pose, gender,
etc. However, $\W$ is still not expressive enough to support faithful
reconstruction of existing images \cite{abdal2019image2stylegan,abdal2020image2stylegan++,richardson2021encoding}.
Therefore, later works mostly consider the \emph{extended} latent
space $\W^{+}$ rather than $\W$ for image editing and face swapping
tasks \cite{tewari2020pie,tov2021designing,xu2022high,zhu2021one}.
An \emph{extended} latent code $w\in\W^{+}$ is a concatenation of
18 different 512-dimensional $w_{i}\in\W$ vectors, one per AdaIN
layer. Clearly, $w\in\W^{+}$ can capture richer style information
than $w'\in\W$.

StyleGAN2 \cite{karras2020analyzing} is an improved version of StyleGAN.
It makes several important changes in the generator architecture and
the training strategy of StyleGAN to account for unpleasant artifacts
in images generated by StyleGAN. Specifically, in StyleGAN2, each
AdaIN layer is replaced with a weight demodulation, and an additional
path length regularization is enforced during training. These modifications
make the StyleGAN2's generator smoother, easier to invert, and able
to generate more real images than the StyleGAN's counterpart while
still retaining full controllability of style via the extended latent
space $\W^{+}$. Therefore, unless stated otherwise, we will consider
StyleGAN2 rather than StyleGAN for the rest of our paper.

\begin{figure*}
\begin{centering}
\includegraphics[width=0.9\textwidth]{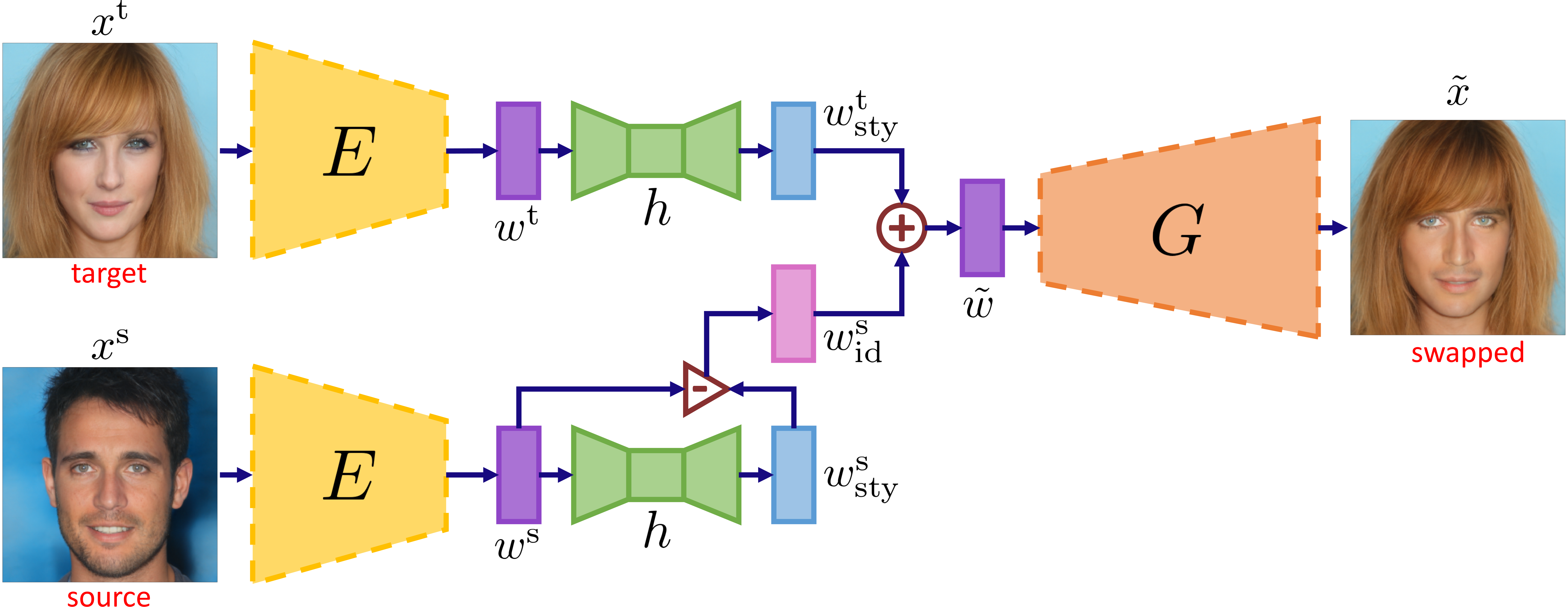}
\par\end{centering}
\caption{Overall framework of our proposed method. $G$ is a StyleGAN2 generator,
$E$ is an encoder, $h$ is a style extractor network. Both $G$ and
$E$ were pretrained, thus, are marked with dashed border.\label{fig:overall_framework}}
\end{figure*}

\subsection{Face Swapping with a Pretrained StyleGAN2\label{subsec:Face-Swapping-with-a-Pretrained-StyleGAN2}}

In the standard face swapping task, we want to transfer the identity
in a source face image $x^{\s}$ to a target face image $x^{\t}$
while keeping other facial attributes (styles) in $x^{\t}$ unchanged.
We use a well pretrained encoder $E$ \cite{tov2021designing} to
embed $x^{\s}$ and $x^{\t}$ into the latent space $\W^{+}$ of a
StyleGAN2 pretrained on face images, and derive two latent codes $w^{\s}$
and $w^{\t}$ $\in\W^{+}$ respectively. 

Motivated by the expressiveness and linear editability of $\W^{+}$,
we argue that we can disentangle $\W^{+}$ into two subspaces $\W_{\id}^{+}$
and $\W_{\sty}^{+}$ characterizing identity and style respectively,
so that a latent code $w\in\W^{+}$ can be represented as follows:
\begin{equation}
w=w_{\id}+w_{\sty}\label{eq:id_style_decomp}
\end{equation}
where $w_{\id}\in\W_{\id}^{+}$ and $w_{\sty}\in\W_{\sty}^{+}$ denote
the \emph{``identity''} and \emph{``style''} latent codes that
capture the identity and style information in $w$. Similarly, we
also have $w^{\s}=w_{\id}^{\s}+w_{\sty}^{\s}$ and $w^{\t}=w_{\id}^{\t}+w_{\sty}^{\t}$.

The above arithmetic decomposition of $w^{\s}$ and $w^{\t}$ suggests
an efficient and effective way to do face swapping which is adding
the source identity code $w_{\id}^{\s}$ and the target style code
$w_{\sty}^{\t}$ together:
\begin{align}
\tilde{w} & =w_{\id}^{\s}+w_{\sty}^{\t}\label{eq:face_swap_1}\\
 & =w^{\s}-w_{\sty}^{\s}+w_{\sty}^{\t}\\
 & =w_{\id}^{\s}+w^{\t}-w_{\id}^{\t}
\end{align}
Here, $\tilde{w}$ denotes the \emph{``swapped''} latent code. We
generate a swapped face image $\tilde{x}$ from $\tilde{w}$ by feeding
$\tilde{w}$ to the StyleGAN2 generator $G$, i.e. $\tilde{x}=G(\tilde{w})$.

Intuitively, we can generalize the face swapping in Eq.~\ref{eq:face_swap_1}
to account for multiple sources and targets as follows:
\begin{align}
\tilde{w} & =\left(\sum_{i=1}^{M}\alpha_{i}w_{\id}^{\s_{i}}\right)+\left(\sum_{j=1}^{N}\beta_{j}w_{\sty}^{\t_{j}}\right)\label{eq:face_swap_multi_1}\\
 & =\sum_{i=1}^{M}\sum_{j=1}^{N}\alpha_{i}\beta_{j}\left(w_{\id}^{\s_{i}}+w_{\sty}^{\t_{j}}\right)\\
 & =\sum_{i=1}^{M}\sum_{j=1}^{N}\alpha_{i}\beta_{j}\tilde{w}^{\s_{i},\t_{j}}\label{eq:face_swap_multi_3}
\end{align}
where $M,N\geq1$ and $\alpha_{1},...,\alpha_{M}\geq0$; $\sum_{i=1}^{M}\alpha_{i}=1$
and $\beta_{1},...,\beta_{N}\geq0$; $\sum_{j=1}^{N}\beta_{j}=1$.
The final swapped code $\tilde{w}$ in Eq.~\ref{eq:face_swap_multi_3}
is equivalent to the linear interpolation of individual swapped codes
derived from all source-target pairs.

The remaining question is how to model $w_{\id}$ and $w_{\sty}$.
There could be various approaches to this but in our paper, we propose
to learn a \emph{style extractor network} $h$ that maps $w$ to $w_{\sty}$,
i.e., $w_{\sty}=h(w)$, and compute $w_{\id}$ as $w_{\id}=w-h(w)$.
We found this solution works well in our experiments so we leave the
exploration of other approaches for future work. We can adopt $h$
into the formula of $\tilde{w}$ as follows:
\begin{equation}
\tilde{w}=w^{\s}-h(w^{\s})+h(w^{\t})\label{eq:face_swap_with_h}
\end{equation}

We refer to our proposed method as \emph{``Arithmetic Face Swapping''}
(AFS), and provide an illustration of it in Fig.~\ref{fig:overall_framework}.
Below, we provide a condition for $h$ which ensures that $\tilde{w}$
preserves the identity in $w^{\s}$.
\begin{prop}
$\tilde{w}_{\id}=w_{\id}^{\s}$ if and only if $h(\tilde{w})=h(w^{\t})$.\label{prop-1}
\end{prop}
\begin{proof}
Since $\tilde{w}_{\id}=\tilde{w}-h(\tilde{w})=w_{\id}^{\s}+h(w^{\t})-h(\tilde{w})$,
we have $\tilde{w}_{\id}=w_{\id}^{\s}$ if and only if $h(\tilde{w})=h(w^{\t})$.
\end{proof}
\begin{cor}
[Chains of face swapping operations]If the style extractor $h$ can
preserve style for \textbf{any} face swapping operation, i.e., $h(\tilde{w})=h(w^{\t})\ \forall w^{\s},w^{\t}$;
then we can perform a chain of face swapping operations without losing
identity of the source image.\label{cor-1}
\end{cor}
\begin{proof}
This is just a direct extension of Prop.~\ref{prop-1} to chains
of face swapping operations and can be easily proven by induction.
Below, we give an example of how it will work. 

Assuming that we have performed the following chain of face swapping
operations for three latent codes $w^{(1)}$, $w^{(2)}$, $w^{(3)}$
based on the decomposition in Eq.~\ref{eq:id_style_decomp}:
\begin{align*}
\tilde{w}^{(1,2)} & =w_{\id}^{(1)}+w_{\sty}^{(2)}\\
\tilde{w}^{(1,3)} & =\tilde{w}_{\id}^{(1,2)}+w_{\sty}^{(3)}
\end{align*}
According to Prop.~\ref{prop-1}, if $h\left(\tilde{w}^{(1,2)}\right)=h\left(w^{(2)}\right)$,
$\tilde{w}_{\id}^{(1,2)}$ will equal $w_{\id}^{(1)}$; and $\tilde{w}^{(1,3)}$
can be written as:
\[
\tilde{w}^{(1,3)}=w_{\id}^{(1)}+w_{\sty}^{(3)}
\]
which means $\tilde{w}^{(1,3)}$ is only composed of the identity
of $w^{(1)}$ and the style of $w^{(3)}$ with \emph{no} style of
$w^{(2)}$. This allows $h\left(\tilde{w}^{(1,3)}\right)$ to be exactly
similar to $h\left(w^{(3)}\right)$ and if this happens, $\tilde{w}_{\id}^{(1,3)}$
will equal $w_{\id}^{(1)}$. The same reasoning is applied when we
perform another face swapping operation between $\tilde{w}^{(1,3)}$
and $w^{(4)}$.
\end{proof}
Given the elegance of $\Model$, one may wonder whether there exits
a setting for $w_{\id}$, $w_{\sty}$, and $h$ so that $\tilde{w}$
computed in Eq.~\ref{eq:face_swap_with_h} can faithfully preserve
the source identity and the target style for any source-target pair
$(w^{\s},w^{\t})$. In the remark below, we show that such setting
exists in theory.
\begin{rem}
[Possibility of $\Model$]If we can design $w\in\Real^{d}$ so that
its first $k$ components only capture identity and its last $d-k$
components only capture style, and if $h$ is a function that extracts
the last $d-k$ components of a vector, then $\tilde{w}=w^{\s}-h(w^{\s})+h(w^{\t})$
can faithfully preserve the source identity and the target style.\label{remark-1}
\end{rem}
It is not hard to verify the correctness of Remark~\ref{remark-1}.
Intuitively, it means we can express $w_{\id}$ as $(w_{1},...,w_{k},0,...,0)$
and $w_{\sty}$ as $(0,...,0,w_{k+1},...,w_{d})$.

Next, we will describe the objectives for training $h$ so that $h$
yields correct identity-style disentanglement and satisfies the style
preservation condition in Prop.~\ref{prop-1}.

\subsection{Objective Functions}

\paragraph{ID loss}

As $\tilde{w}$ contains the identity part of $w^{\s}$, we want its
generated image $G(\tilde{w})$ will preserve the identity in the
source image $x^{\s}$. We define the \emph{identity similarity score}
between $G(\tilde{w})$ and $x^{\s}$ as the cosine similarity between
the representations computed by a pretrained ArcFace \cite{deng2019arcface}
model $R$ for $G(\tilde{w})$ and $x^{\s}$, and minimize the loss
below:
\begin{equation}
\Loss_{\text{ID}}=1-\cos\left(R\left(G\left(\tilde{w}\right)\right),R\left(x^{\s}\right)\right)\label{eq: loss-id}
\end{equation}
where $\cos(\cdot,\cdot)$ denotes the cosine similarity.

\begin{table}
\begin{centering}
\resizebox{.98\columnwidth}{!}{%
\begin{tabular}{ccccc}
\hline 
\multirow{1}{*}{Method} & \multicolumn{1}{c}{ID ($\uparrow$)} & \multirow{1}{*}{Expr.($\downarrow$)} & \multirow{1}{*}{Pose($\downarrow$)} & FID ($\downarrow$)\tabularnewline
\hline 
\hline 
FSGAN \cite{nirkin2019fsgan}  & 0.20 & 6.80 & \textbf{4.31} & 67.00\tabularnewline
FaceShifter \cite{li2019faceshifter} & 0.48 & 7.15 & 5.52 & 12.16\tabularnewline
MegaFS \cite{zhu2021one} & 0.46 & 8.67 & 6.98 & 11.23\tabularnewline
$\Model$ (Ours) & \textbf{0.49} & \textbf{5.01} & 4.54 & \textbf{4.56}\tabularnewline
\hline 
\end{tabular}}
\par\end{centering}
\caption{Quantitative results of our face swapping method and baselines on
the CelebA-HQ test set.\label{tab:quantitative-results}}
\end{table}

\paragraph{Feature map loss}

Intermediate feature maps of a StyleGAN2 generator can be roughly
divided into three groups with low (4$\times$4, 8$\times$8), middle
(16$\times$16, 32$\times$32) and high (64$\times$64 $\rightarrow$
1024$\times$1024) spatial resolutions. The first two groups have
been shown to represent coarse facial attributes such as pose, face
shape, and general hair style that have little effect on identity
\cite{karras2019style}. In addition, since each feature map $i$
can be modified separately via each component $w_{i}$ of a latent
code $w\in\W^{+}$, we can force $\tilde{w}$ to create the same style
as $w^{\t}$ at low-resolution feature maps, leaving the identity
information from $w^{\s}$ to be represented by high-resolution feature
maps. This can be done by minimizing the loss below:
\begin{equation}
\Loss_{\text{feat}}=\left\Vert G_{\text{f}}\left(\tilde{w}\right)-G_{\text{f}}\left(w^{\t}\right)\right\Vert _{2}\label{eq:feat_loss}
\end{equation}
where $G_{\text{f}}(\cdot)$ denotes the (second) 32$\times$32 feature
map output of the StyleGAN2 generator.

\paragraph{Perceptual loss}

In addition to $\Loss_{\text{feat}}$, we also make use of the Learned
Perceptual Image Patch Similarity (LPIPS) loss \cite{zhang2018perceptual}
to enforce the perceptual similarity between $G(\tilde{w})$ and $x^{\t}$.
LPIPS has been shown to account for many nuances of human perception
better than other perceptual losses (e.g., SSIM \cite{wang2004image},
FSIM \cite{zhang2011fsim}). The formula of our LPIPS loss is given
as follows:
\begin{equation}
\Loss_{\text{LPIPS}}=\left\Vert \text{VGG}\left(G\left(\tilde{w}\right)\right)-\text{VGG}\left(G\left(w^{\t}\right)\right)\right\Vert _{2}\label{eq:perceptual_loss}
\end{equation}
where VGG$(\cdot)$ denotes the feature vector extracted from a pretrained
VGG \cite{simonyan2014very}.

\paragraph{Consistency loss}

To ensure $h(\tilde{w})$ equals $h(w^{\t})$ in Prop.~\ref{prop-1},
we simply minimize the L1 distance between $h(\tilde{w})$ and $h(w^{\t})$
as follows:
\begin{equation}
\Loss_{\text{cons}}=\left\Vert h\left(\tilde{w}\right)-h\left(w^{\t}\right)\right\Vert _{1}\label{eq:consistency_loss}
\end{equation}

\paragraph{Overall loss}

The overall loss function for training $h$ is:
\begin{equation}
\mathcal{L}=\lambda_{1}\mathcal{L}_{\text{ID}}+\lambda_{2}\mathcal{L}_{\text{feat}}+\lambda_{3}\mathcal{L}_{\text{LPIPS}}+\lambda_{4}\Loss_{\text{cons}}\label{eq:final_loss}
\end{equation}
where $\lambda_{1}$, $\lambda_{2}$, $\lambda_{3}$, $\lambda_{4}$
$\geq$ 0 are coefficients. We note that we keep VGG and ArcFace intact
and only train $h$.

\subsection{Architecture of the Style Extractor Network}

\begin{figure}
\begin{centering}
\includegraphics[width=1\columnwidth]{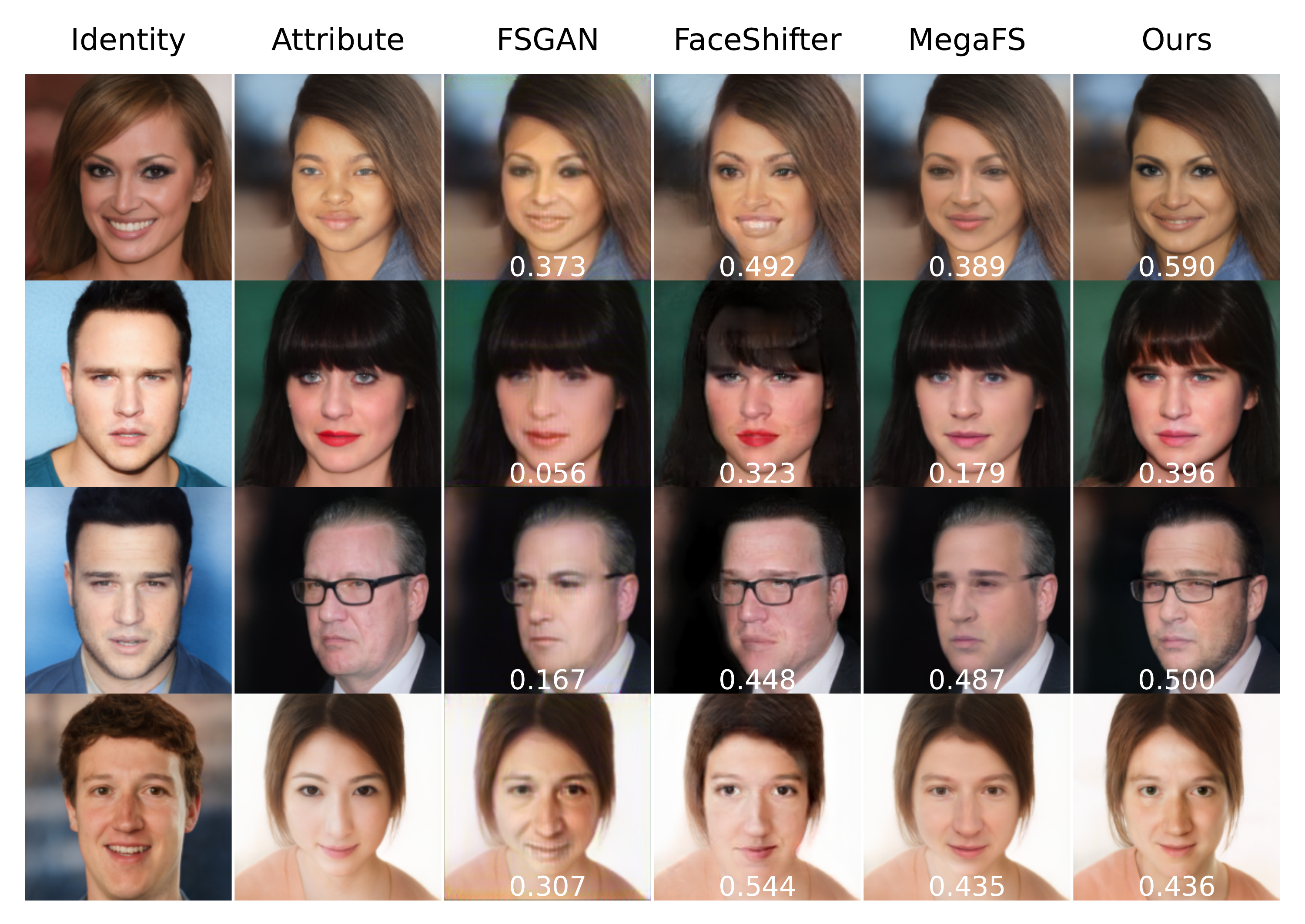}
\par\end{centering}
\caption{Visualization of swapped faces produced by our method and baselines.
The first two columns contain source faces (``identity'') and target
faces (``attribute'') from the CelebA-HQ test set. The identity
similarity score is shown at the bottom of each swapped face image.\label{fig:qualitative-comparision.}}
\end{figure}

Since a latent code $w\in\W^{+}$ has many dimensions (18$\times$512$=$9216
in total), if we feed the whole latent code $w$ to the style extractor
$h$, $h$ will be very big and difficult to learn. From the observation
that the components $w_{1},...,w_{18}$ of $w$ can be independent
of each other and usually capture different styles typical of the
feature maps they control, we instead use 18 different style extractor
sub-networks $h_{1},...,h_{18}$ for 18 different components of $w$.
For each sub-network $h_{i}$, we have the input vector of size 512
which will be transformed to a hidden vector of size 256 via a fully
connected (FC) layer. Then, this hidden vector will be fed to a series
of highway layers \cite{srivastava2015training} before being mapped
back to the size 512. By default, the number of highway layers is
2. We empirically found that this architecture facilitates the learning
of $h_{i}$. We tried to replace the highway layers with standard
FC layers but achieved worse results. We hypothesize that may be the
gating mechanism in highway layers facilitate selecting which elements
of $w_{i}$ should be discarded and retained to get the correct style
code component.

%% file: exp.tex
\begin{figure}
\begin{centering}
\includegraphics[width=1\columnwidth]{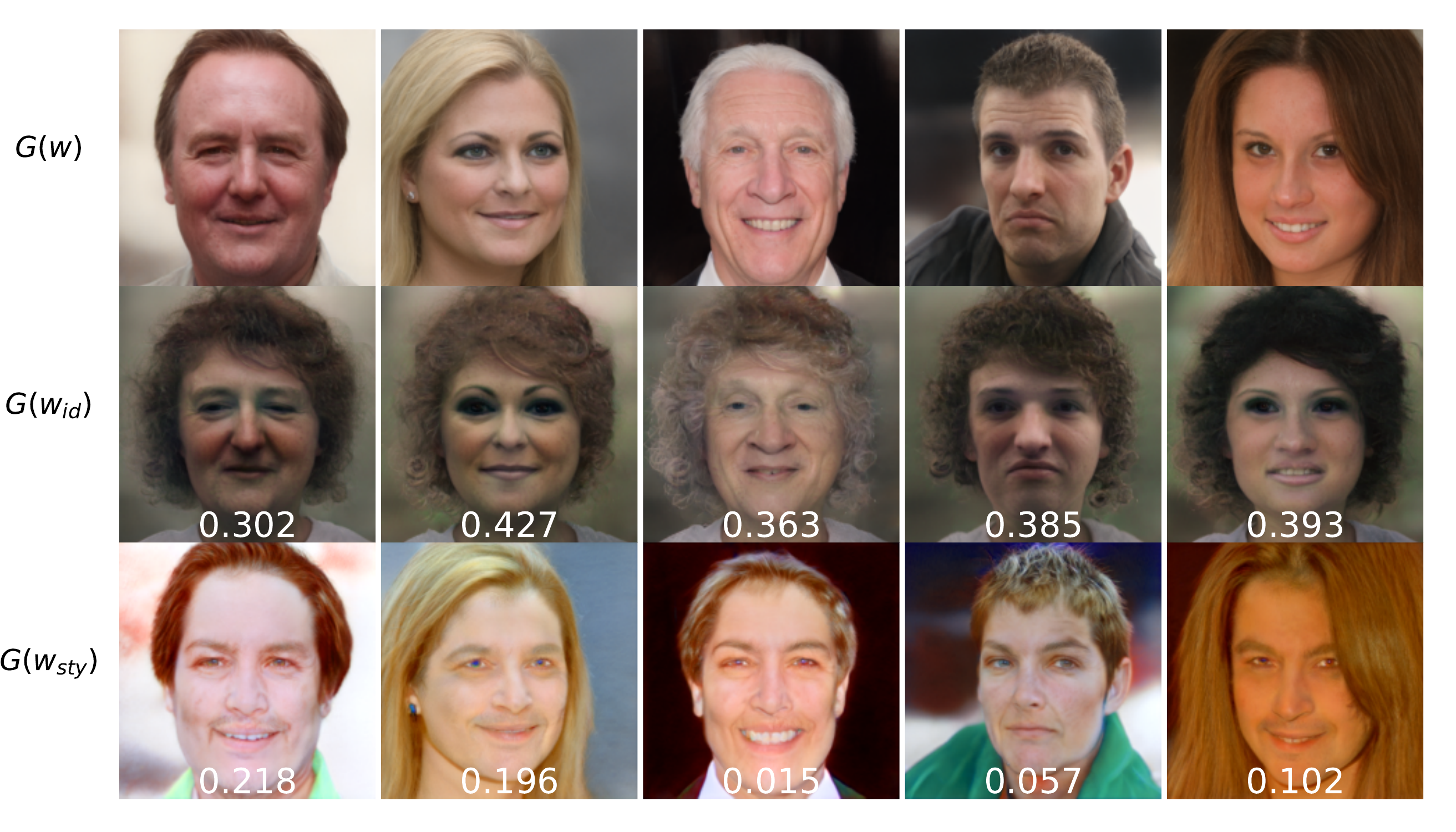}
\par\end{centering}
\caption{Visualization of the \emph{original}, \emph{``identity''}, and \emph{``style''
images} ($G(w)$, $G(w_{\protect\id})$, $G(w_{\protect\sty})$) generated
from the \emph{original}, \emph{``identity''} and \emph{``style''
latent codes} ($w$, $w_{\protect\id}$, $w_{\protect\sty}$).\label{fig:plot-w-id}}
\end{figure}

\subsection{Experimental Setup}

\subsubsection{Datasets and training settings}

We used the CelebA-HQ dataset \cite{karras2017progressive} for our
experiments. This dataset contains 30k celebrities' faces at megapixel
resolutions (1024$\times$1024) and is challenging for face swapping.
We used a pretrained e4e encoder \cite{tov2021designing} to invert
27176 images in the training and validation sets of CelebA-HQ to latent
codes, and trained our style extractor $h$ on these latent codes.
We trained $h$ for 10 epochs with a batch size of 4. We used an Adam
optimizer with an initial learning rate of 1e-4, decayed to 1e-6 following
a cosine schedule over all steps. We set the loss coefficients (Eq.~\ref{eq:final_loss})
as follows: $\lambda_{1}$=1, $\lambda_{2}$=3.5, $\lambda_{3}$=1,
and $\lambda_{4}$=0.1.

\subsubsection{Metrics}

For identity preservation evaluation, we use the \emph{identity similarity
score} which is the cosine similarity produced by a pretrained ArcFace
\cite{deng2019arcface} for $\tilde{x}$ and $x^{\s}$. We follow
previous works \cite{li2019faceshifter,zhu2021one} and consider \emph{pose}
and \emph{expression} for attribute preservation evaluation. The \emph{expression
error} is the L2 distance between 2D landmarks of $\tilde{x}$ and
$x^{\t}$ which are extracted by the Face Alignment library\footnote{\url{https://github.com/1adrianb/face-alignment}}.
The \emph{pose error} is the L2 distance between the outputs of Hopenet
\cite{ruiz2018fine} w.r.t. $\tilde{x}$ and $x^{\t}$. We computed
the identity similarity score, expression error, pose error on 5000
source-target pairs sampled randomly from the test set of CelebA-HQ.
To measure the visual quality of swapped faces, we use the Frechet
Inception Distance (FID) \cite{heusel2017gans}. We formed 15k source-target
pairs from 30k images of CelebA-HQ and generated 15k swapped faces
from these pairs. We computed the FID between the swapped faces images
and the original images of CelebA-HQ.

\subsection{Main Results}

\begin{figure*}
\begin{centering}
\includegraphics[width=0.8\textwidth]{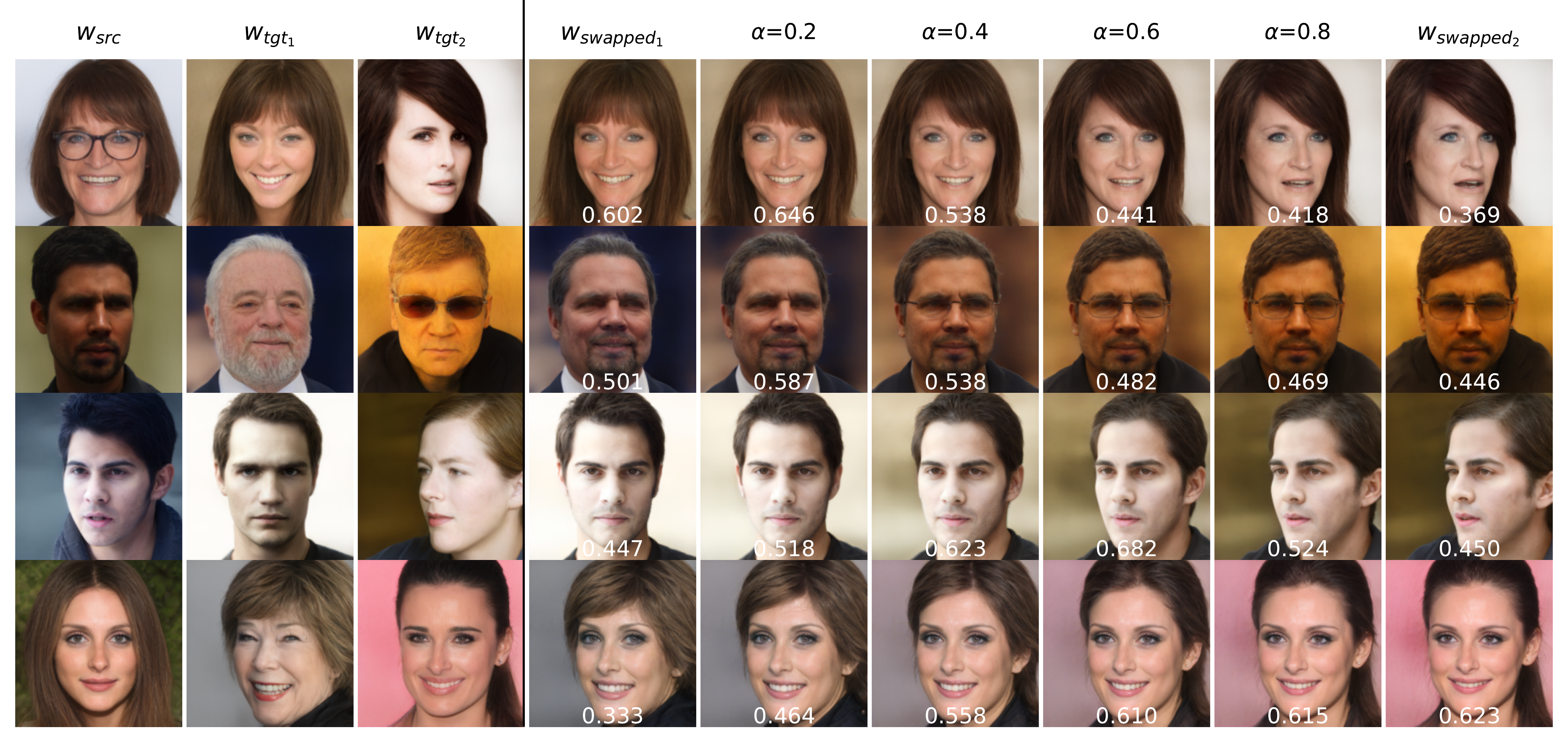}
\par\end{centering}
\caption{Results of combining an identity from one source ($w^{\protect\s}$)
with styles from two targets ($w^{\protect\t_{1}}$, $w^{\protect\t_{2}}$).
The formula of the final swapped latent code $\tilde{w}$ is $\tilde{w}=w_{\protect\id}^{\protect\s}+\left(\alpha w_{\protect\sty}^{\protect\t_{1}}+(1-\alpha)w_{\protect\sty}^{\protect\t_{2}}\right)$
with $0\protect\leq\alpha\protect\leq1$.\label{fig:one_id_two_stys}}
\end{figure*}

\subsubsection{Quantitative comparison}

As shown in Table~\ref{tab:quantitative-results}, $\Model$ outperforms
all face swapping baselines in all metrics except for the case in
which ours achieves a higher pose error than FSGAN. This means the
swapped faces generated by our method not only have better visual
quality but also preserve the source identity and target attributes
better than those generated by the baselines, which justifies the
effectiveness of our arithmetic identity-style disentanglement of
the latent space $\W^{+}$ for face swapping. For the case that $\Model$
has higher pose error than FSGAN, we think the possible reason is
that FSGAN explicitly models the target pose (via its 3DMM) and uses
this information for face swapping while our method does not. However,
since FSGAN does not use GAN to generate swapped faces, it achieves
a very poor FID score compared to ours and other GAN-based baselines.

\subsubsection{Visualization of swapped faces}

As shown in Fig.~\ref{fig:qualitative-comparision.}, the swapped
faces generated by $\Model$ look more natural and detailed than those
generated by the baselines, which reflects the low FID score of our
method in Table~\ref{tab:quantitative-results}. Since FSGAN and
MegaFS apply the target face mask to the swapped result to keep the
hair and background of the target face image intact, it is not surprising
that their swapped faces look exactly similar to the target face in
all details except for the face region specified by the face mask.
However, simply using face masks as a postprocessing step can lead
to some unwanted artifacts, e.g., when the target subject wears eyeglasses
(row 3). Our method, despite not using face masks, still preserves
the target attributes very well. FaceShifter also does this well since
it only adjusts certain areas in the target feature maps specified
by the source identity embedding vector in a selective manner. However,
FaceShifter often produces artifacts when the source and target faces
have largely different shapes (row 1 and row 4) while our method does
not have such problem.

\subsubsection{Visualization of identity and style latent codes}

One interesting capability of $\Model$ that may not be found in other
face swapping methods is the visualization of identity and style codes.
Fig.~\ref{fig:plot-w-id} shows some examples of \emph{``identity''}
and \emph{``style'' images} ($G(w_{\id})$, $G(w_{\sty})$) generated
from the corresponding identity and style codes ($w_{\id}$, $w_{\sty}$),
respectively. All these images are interpretable faces, which empirically
verifies that $w_{\id}$ and $w_{\sty}$ are in the subspaces of $\W^{+}$
rather than some arbitrary spaces. The ``identity'' images have
almost the same general styles (e.g., straight-looking, brown curly
hair, pale skin) and only differ in some identity-specific facial
attributes (e.g., eye shape, nose shape, cheek type) which, we think,
are considered useful by the face recognition model $R$. By contrast,
the ``style'' images seem to have similar identities, and are mainly
different in identity-unrelated facial attributes (e.g., hair style
and color, skin color, shirt color). These results indicate that our
method properly disentangles a latent code into identity and style
parts, and is, to some extent, more interpretable than existing methods.

\subsubsection{Face swapping with one identity and two styles}

In Fig.~\ref{fig:one_id_two_stys}, we show face swapping results
with an identity from one source and styles from two targets. The
contribution of each target to the final result is controlled via
a hyperparameter $\alpha$ ($0\leq\alpha\leq1$). Clearly, the style
of the final swapped face is a linear interpolation of the styles
of two individual swapped faces, which empirically justifies our formula
in Eq.~\ref{eq:face_swap_multi_3}. We also note that there is a
slight difference in the identity similarity scores of the two individual
swapped faces. It implies that the target style has some effect on
the output of the face recognition model $R$.

\begin{figure*}
\begin{centering}
\includegraphics[width=0.75\textwidth]{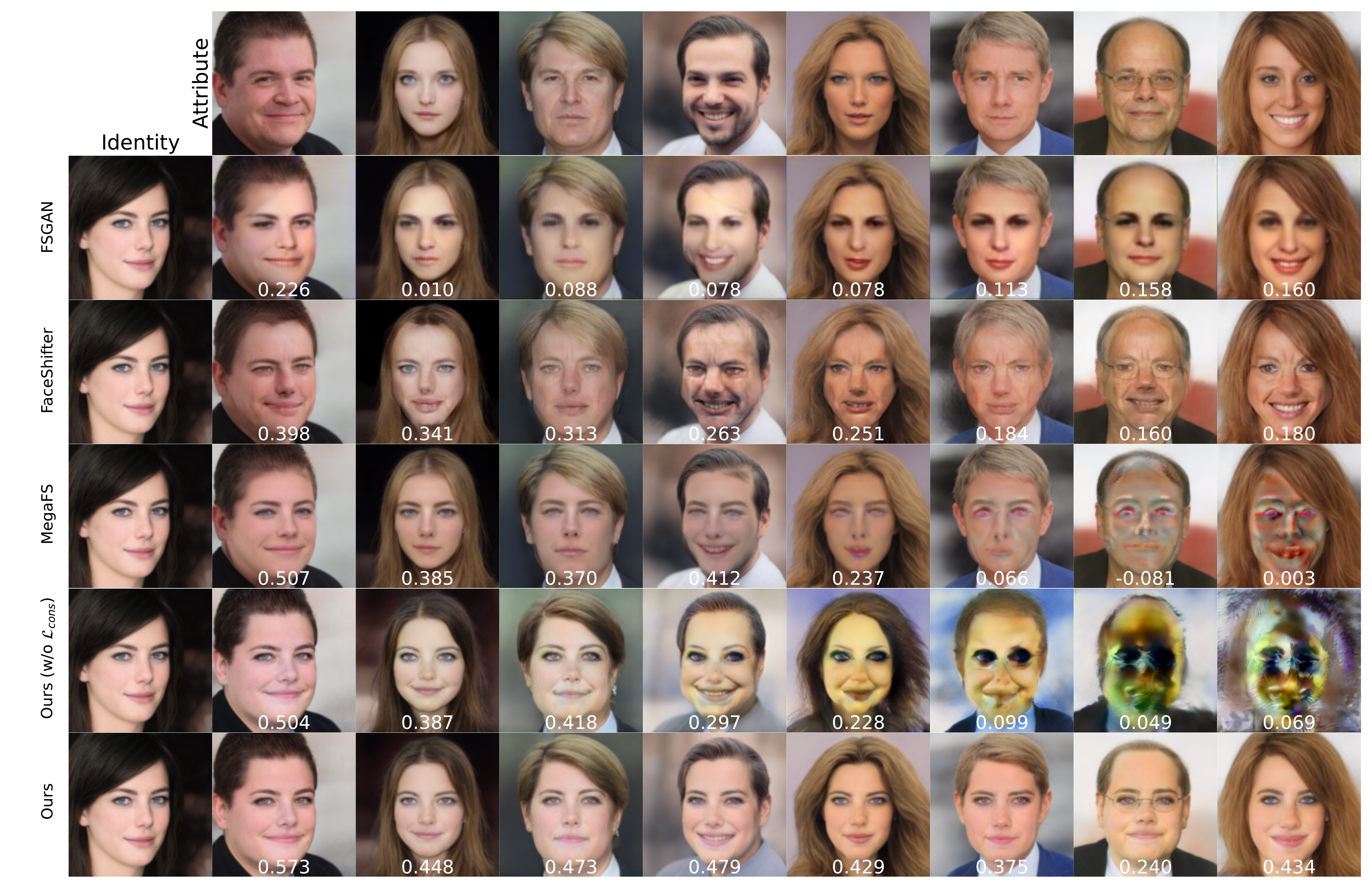}
\par\end{centering}
\caption{Results of running a chain of face swapping operations. The leftmost
column shows a single source face, and the topmost row shows multiple
target faces that will be used sequentially for face swapping. Each
of the following rows shows a sequence of swapped faces generated
by each method. Each swapped face in the sequence is derived from
the face \emph{immediately to the left} as source and the face \emph{at
the topmost row} as target. The identity similarity score between
a swapped face and the original source face (leftmost) is shown at
the bottom of the swapped face image.\label{fig:chain-of-decompositions.}}
\end{figure*}

\subsubsection{Chains of face swapping operations}

In the Method section, we theoretically showed that our method is
able to perform chains of face swapping operations if $h$ satisfies
the condition in Cor.~\ref{cor-1} which is characterized by the
loss $\Loss_{\text{cons}}$ (Eq.~\ref{eq:consistency_loss}). In
this experiment, we empirically verify this capability. From Fig.~\ref{fig:chain-of-decompositions.},
it is clear that adding $\Loss_{\text{cons}}$ to the final loss greatly
improves the performance of our method, allowing our method to generate
high-fidelity swapped faces that well preserve the source identity
and target attributes even after many successive steps of face swapping.
Meanwhile, FaceShifter and MegaFS start producing unnatural swapped
faces after just two and four steps, respectively.

\subsubsection{Other results}

Due to space limit, we refer readers to the Appendix for results on
``face swapping with two identities and one style'', and on ``semantically
meaningful linear transformations in the identity and style spaces''.

%% file: discuss.tex
We have proposed a novel high-fidelity face swapping method called
``Arithmetic Face Swapping'' (AFS) that disentangles the intermediate
latent space $\W^{+}$ of a pretrained StyleGAN2 into two subspaces
w.r.t. the identity and style, and treats face swapping as an arithmetic
operation in this space, i.e., the summation of a source identity
code and a target style code. Our method can easily generalize over
the standard face swapping to support multiple source identities and
target styles. Extensive experiments have demonstrated the advantages
of our method in generating high-quality swapped faces over several
SOTAs. In the future, we would like to examine different architectures
of $h$ and different training objectives that lead to better preservation
of the source identity and target attributes.

%% file: ethic.tex
\subsection{General Ethical Conduct}

The CelebA-HQ dataset \cite{karras2017progressive} is available for
non-commercial research purposes. When using this dataset, we adhere
to the agreements specified by the dataset owners\footnote{https://mmlab.ie.cuhk.edu.hk/projects/CelebA.html}.
Since the data we used are only human face images with no identity
or facial attribute annotations, and since the face images are publicly
available on the Internet, it is unlikely that these data reveal any
personally identifiable information such as name, age, sex, etc.,
or contain private, sensitive information that could be degrading
or embarrassing for some people.

\subsection{Potential Negative Societal Impacts}

Our proposed method is aimed to be a simple yet effective method for
face swapping that could be widely applied to the entertainment industry
and identity-concealing tasks. Despite these good purposes, our method
may be misused to cause societal harms. Some potential scenarios are: 
\begin{enumerate}
\item Counterfeiting a picture of a person $A$ doing bad things by swapping
the identity of the true person in the picture with that of $A$ to
damage the reputation of $A$.
\item Creating a face with a non-existing identity (e.g., face swapping
with multiple identities) and use this identity to do frauds without
being tracked by the authorities.
\end{enumerate}
The first scenario raises a harassment concern while the second one
raises a security concern. Both scenarios are difficult to mitigate
and tend be more serious when our proposed method becomes better.
For the first scenario, a possible solution is building a framework
that allows face swapping only if the target image contains no actions
that break the social and moral norms. For the second scenario, one
solution is maintaining a database of all known faces for verifying
whether a given face exists in the database or not.

%% file: appendix.tex
\subsection{Ablation Studies}

\subsubsection{Using different feature layers of the StyleGAN2 generator\label{subsec:trade-off-middle-layer}}

Table~\ref{tab:different-layer-of-gen} shows the quantitative results
when different feature layers are used for computing $\Loss_{\text{feat}}$.
We keep all the loss coefficients as default. We observe that there
is a clear trade-off between style and identity preservation. As we
use a deeper layer for our loss (convs.7), the identity similarity
decreases while the preservation of expression and pose is better.
This is expected as using deeper layer will force the $\tilde{w}$
close to $w^{\t}$. By contrast, using lower layer (convs.3) will
cause style to be not well preserved. The best feature layer, in our
opinion, is convs.5.

\begin{table}
\begin{centering}
\resizebox{.99\columnwidth}{!}{%
\begin{tabular}{ccccc}
\hline 
Feat. layer & ID ($\uparrow$) & Expr. ($\downarrow$) & Pose ($\downarrow$) & FID ($\downarrow$)\tabularnewline
\hline 
\hline 
convs.3 & 0.56 & 7.75 & 8.29 & 5.10\tabularnewline
\emph{convs.5} & \emph{0.49 } & \emph{5.01} & \emph{4.54} & \emph{4.56}\tabularnewline
convs.7 & 0.46 & 4.00 & 3.43 & 4.32\tabularnewline
\hline 
\end{tabular}}
\par\end{centering}
\caption{Quantitative results with different feature layers for computing $\protect\Loss_{\text{feat}}$
(Eq.~\ref{eq:feat_loss}). convs.3, convs.5, convs.7 have feature
maps of size 16$\times$16, 32$\times$32, 64$\times$64, respectively.
The default feature layer is convs.5 highlighted in \emph{italic}.
\label{tab:different-layer-of-gen}}
\end{table}

\subsubsection{Contributions of the loss components}

In Fig.~\ref{fig:ablation-loss}, we show the qualitative results
of our method without $\mathcal{L}_{\text{feat}}$ or $\mathcal{L}_{\text{LPIPS}}$
to examine the contribution of each loss term. Apparently, without
$\mathcal{L}_{\text{feat}},$our method fails to preserve the general
shape of the target face. Meanwhile, without $\mathcal{L}_{\text{LPIPS}}$,
our method fails to maintain the skin color and illumination of the
target face. Therefore, we need both $\mathcal{L}_{\text{feat}}$
and $\mathcal{L}_{\text{LPIPS}}$ to achieve good results.

\begin{figure*}
\begin{centering}
\includegraphics[width=0.99\textwidth]{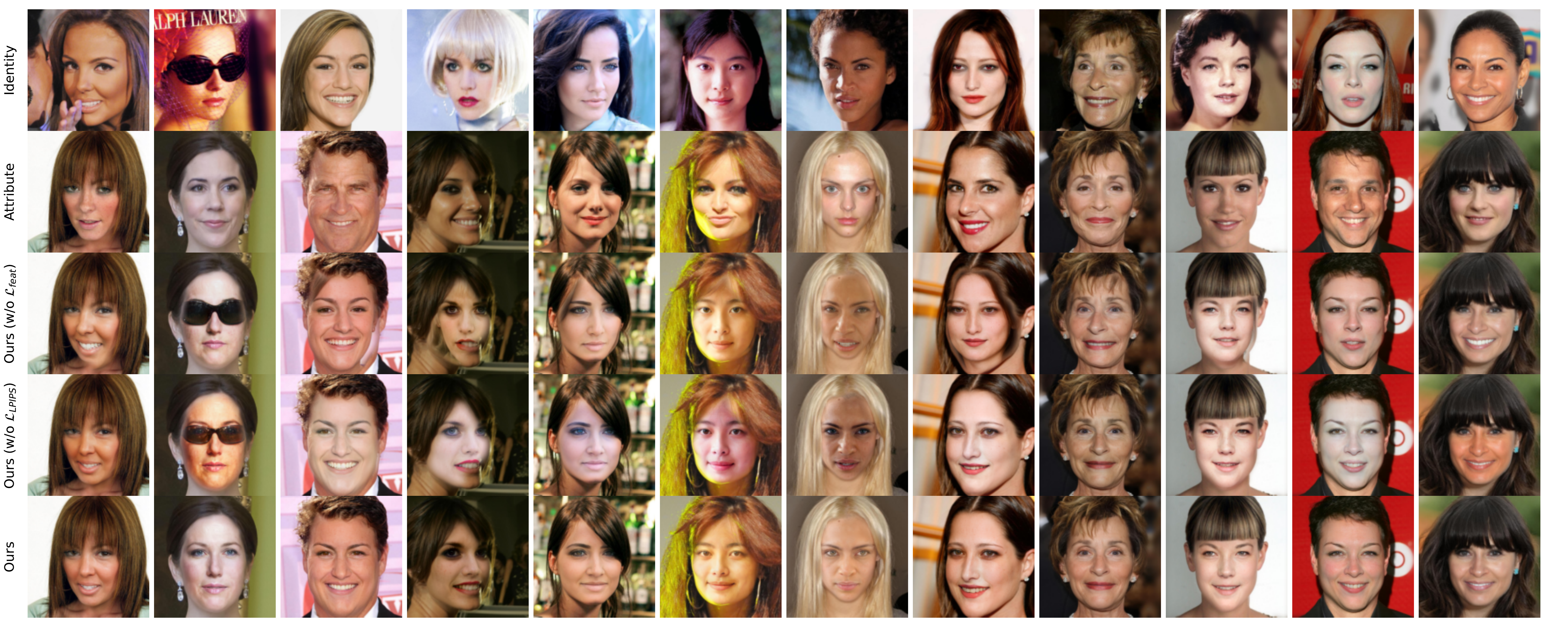}
\par\end{centering}
\caption{Qualitative results of our method without $\protect\Loss_{\text{feat}}$
or $\protect\Loss_{\text{LPIPS}}$.\label{fig:ablation-loss}}
\end{figure*}

\subsection{Other Results}

\subsubsection{More qualitative face swapping results}

We provide more qualitative results of our method in Fig.~\ref{fig:more-qualitative-results}.
Overall, our method preserves the source identity and the target styles
quite well, and can handle occlusions such as wearing glasses. However,
there are cases in which $\Model$ produces undesirable results. For
example, our method is unable to preserve the skin texture of the
target face if it is very different from the counterpart of the source
face.

\begin{figure*}
\begin{centering}
\includegraphics[width=0.99\textwidth]{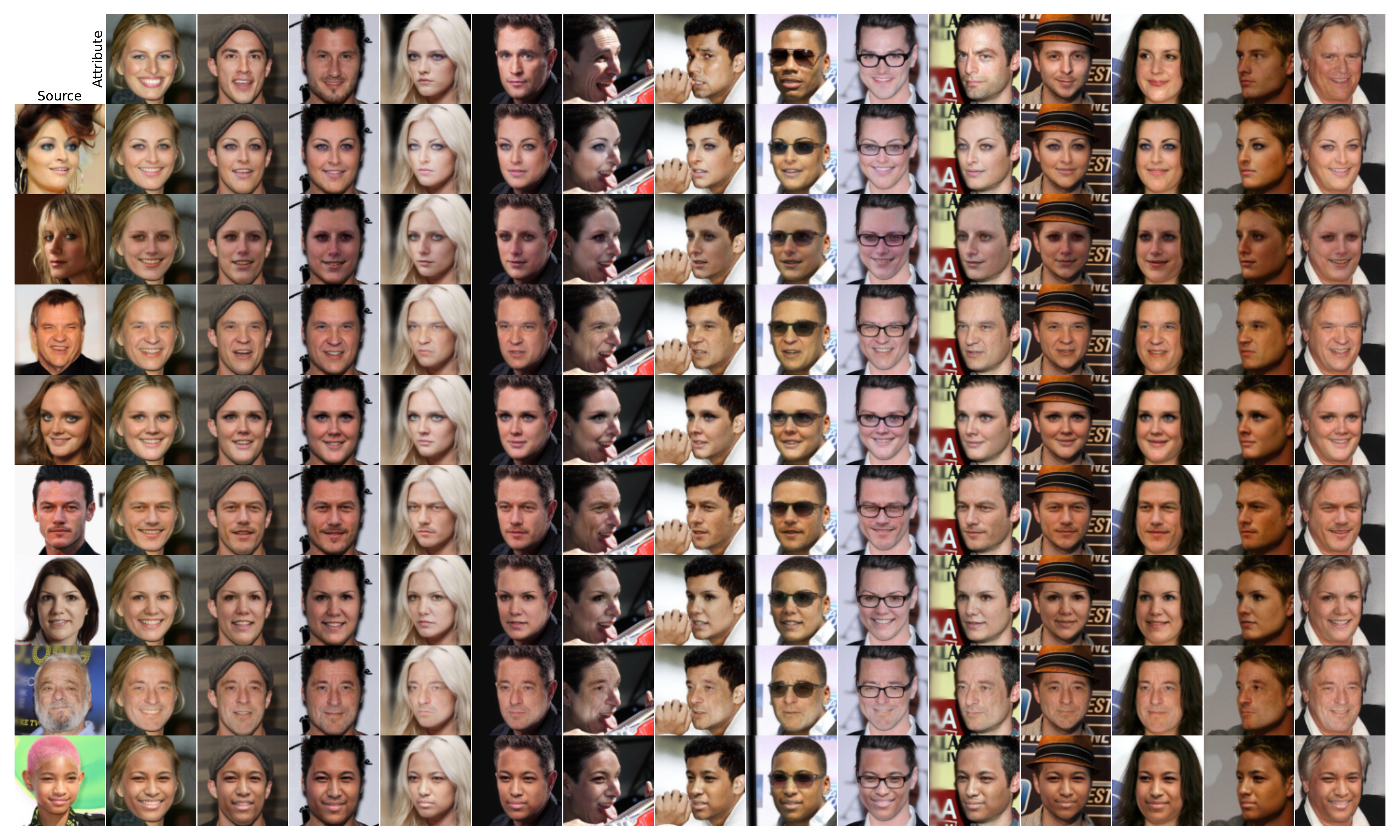}
\par\end{centering}
\caption{More face swapping results of our method with applied masks.\label{fig:more-qualitative-results}}
\end{figure*}

\begin{figure*}
\begin{centering}
\includegraphics[width=0.8\textwidth]{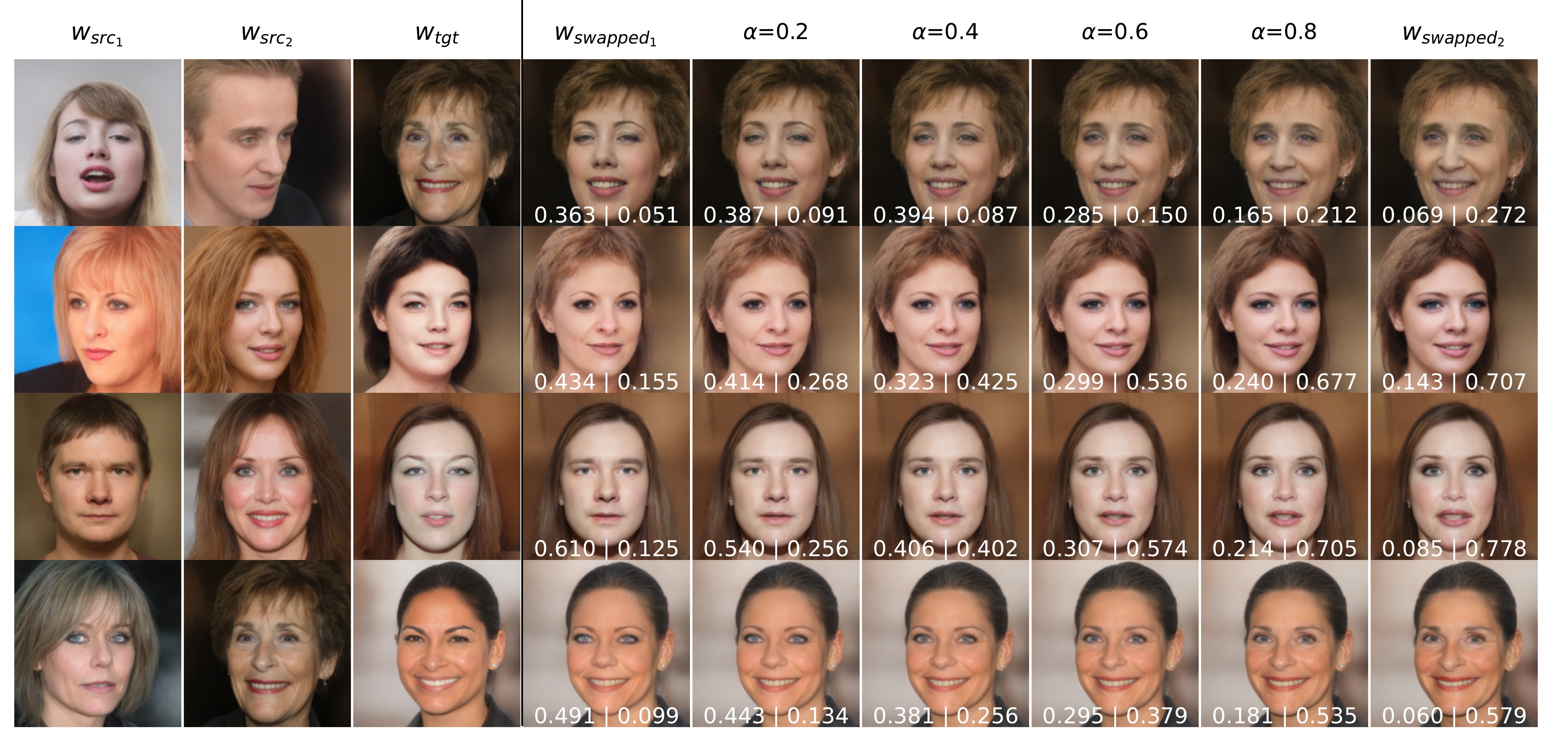}
\par\end{centering}
\caption{Results of combining the identities of two sources ($w^{\protect\s_{1}}$,
$w^{\protect\s_{2}}$) with the style of one target ($w^{\protect\t}$).
The formula of the final swapped latent code $\tilde{w}$ is $\tilde{w}=\left(\alpha w_{\protect\id}^{\protect\s_{1}}+(1-\alpha)w_{\protect\id}^{\protect\s_{2}}\right)+w_{\protect\sty}^{\protect\t}$
with $0\protect\leq\alpha\protect\leq1$. Under each swapped image,
we show two identity similarity scores. The left one corresponds to
the first source, and the right one corresponds to the second source.\label{fig:two_ids_one_sty}}
\end{figure*}

\subsubsection{Face swapping with two identities and one style}

In Fig.~\ref{fig:two_ids_one_sty}, we show some results of combining
the identities of two sources with the style of one target. The hyperparameter
$\alpha$ ($0\leq\alpha\leq1$) controls how much the two sources
affect the final results. When $\alpha$ varies in {[}0, 1{]}, we
can clearly see the identity of the final swapped face interpolates
between the identities of two individual swapped faces while the style
of the final swapped face remains almost unchanged. These results,
again, empirically validate our formula in Eq.~\ref{eq:face_swap_multi_3}.
They also indicate that our method truly separates the identity from
styles. 

\begin{figure*}
\begin{centering}
\includegraphics[width=0.8\textwidth]{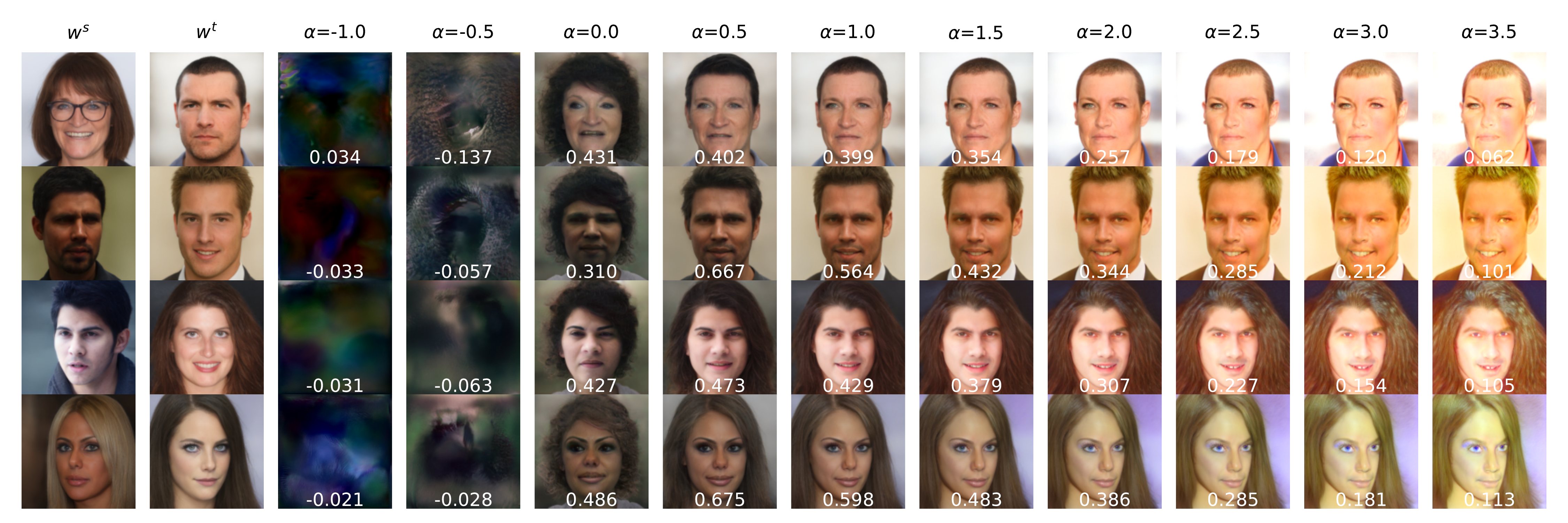}
\par\end{centering}
\caption{Generated swapped faces when the style code is scaled by $\alpha$.
The formula of the swapped latent code is $\tilde{w}=w_{\protect\id}+\alpha w_{\protect\sty}$.
The identity similarity score is shown under each swapped image.\label{fig:scale_sty}}
\end{figure*}

\begin{figure*}
\begin{centering}
\includegraphics[width=0.8\textwidth]{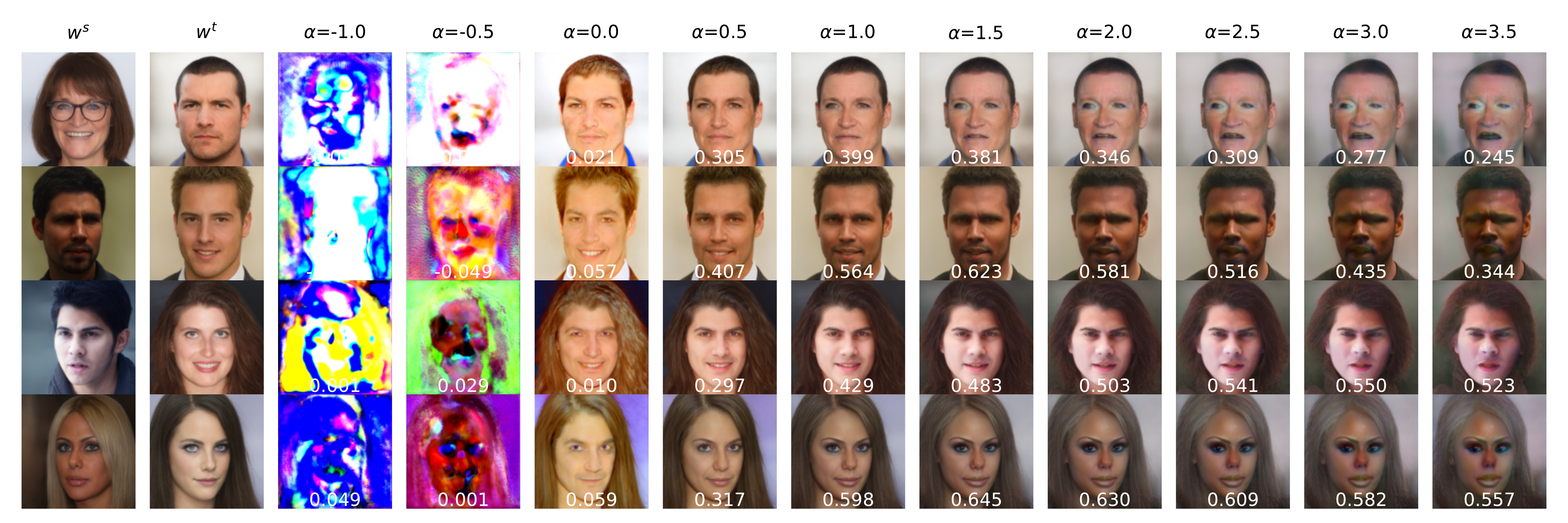}
\par\end{centering}
\caption{Generated swapped faces when the identity code is scaled by $\alpha$.
The formula of the swapped latent code is $\tilde{w}=\alpha w_{\protect\id}+w_{\protect\sty}$.
The identity similarity score is shown under each swapped image.\label{fig:scale_id}}
\end{figure*}

\subsubsection{Scaling identity and style latent codes}

In Figs.~\ref{fig:scale_sty} and \ref{fig:scale_id}, we visualize
swapped faces with the style code and the identity code being scaled,
respectively. It is clear that the swapped faces do not look real
when the style code or the identity code are scaled with negative
or very large $\alpha$. This means the swapped latent code $\tilde{w}$
in these cases no longer resides in the latent space $\W^{+}$. Negative
$\alpha$ causes $\tilde{w}$ to be much more off-the-space-$\W^{+}$
than very large $\alpha$. Scaling style codes leads to the decrease
in identity similarity scores since the swapped code $\tilde{w}$
is now dominated by the scaled style code $\alpha w_{\sty}$. By contrast,
scaling identity codes cause the identity similarity to change slightly,
sometimes even increase compared to not scaling. This means i) the
swapped code $\tilde{w}$ is now dominated by the scaled identity
code $\alpha w_{\sty}$, and ii) the identity code computed by our
method truly reflects the identity specified by the face recognition
model $R$.

\subsubsection{Semantically meaningful linear transformations in the identity and
style spaces}

In this subsection, we empirically examine that with our arithmetic
decomposition of $\W^{+}$ into the identity and style subspaces,
whether we can perform linear transformation in one subspace without
affecting the other or not. Mathematically, it means whether $\left(w_{\id}+\alpha\Delta_{\id}\right)+w_{\sty}$
still has the same styles as $w_{\id}+w_{\sty}$ or not ($\alpha\in\Real$,
$\Delta_{\id}$ is a semantically meaningful editing direction in
$\W_{\id}^{+}$). Note that this only holds true if $\W_{\id}^{+}$
and $\W_{\sty}^{+}$ are fully disentangled as specified in Remark~\ref{remark-1}.
Otherwise, we can write $\left(w_{\id}+\alpha\Delta_{\id}\right)+w_{\sty}$
as $w_{\id}+\left(\alpha\Delta_{\id}+w_{\sty}\right)$, and thus,
$w_{\ensuremath{\sty}}$ will be affected by $\alpha\Delta_{\id}$.

We follow the strategy in GANSpace \cite{harkonen2020ganspace} to
find a set of editing directions $\left\{ \Delta_{\id}\right\} $
($\left\{ \Delta_{\sty}\right\} $) in $\W_{\id}^{+}$ ($\W_{\sty}^{+}$).
Specifically, we compute the principal components from the identity
(style) latent codes of the training face images via PCA, and regard
these principal components as editing directions. In Fig.~\ref{fig:pca_change_id}
(Fig.~\ref{fig:pca_change_sty}), we show the top three principal
components in $\W_{\id}^{+}$ ($\W_{\sty}^{+}$). We observe that
gender, age, skin color are factors that cause large changes of identity
while hair styles, viewpoints, wearing glasses are those that cause
large changes of style. We also see that changing style does not affect
identity much in terms of visual perception (Fig.~\ref{fig:pca_change_sty}),
which suggests that our method can effectively separate identity from
style.

\begin{figure*}
\begin{centering}
\includegraphics[width=0.9\textwidth]{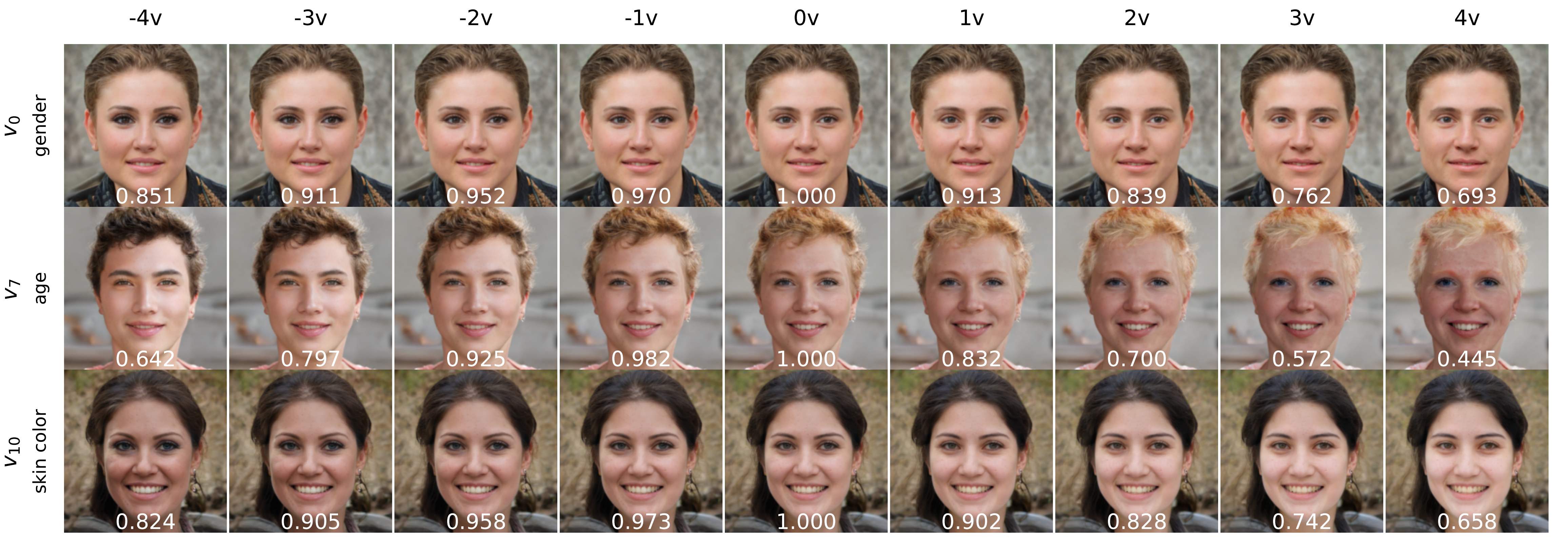}
\par\end{centering}
\caption{The top three principal components in the identity space $\protect\W_{\protect\id}^{+}$.\label{fig:pca_change_id}}
\end{figure*}

\begin{figure*}
\begin{centering}
\includegraphics[width=0.9\textwidth]{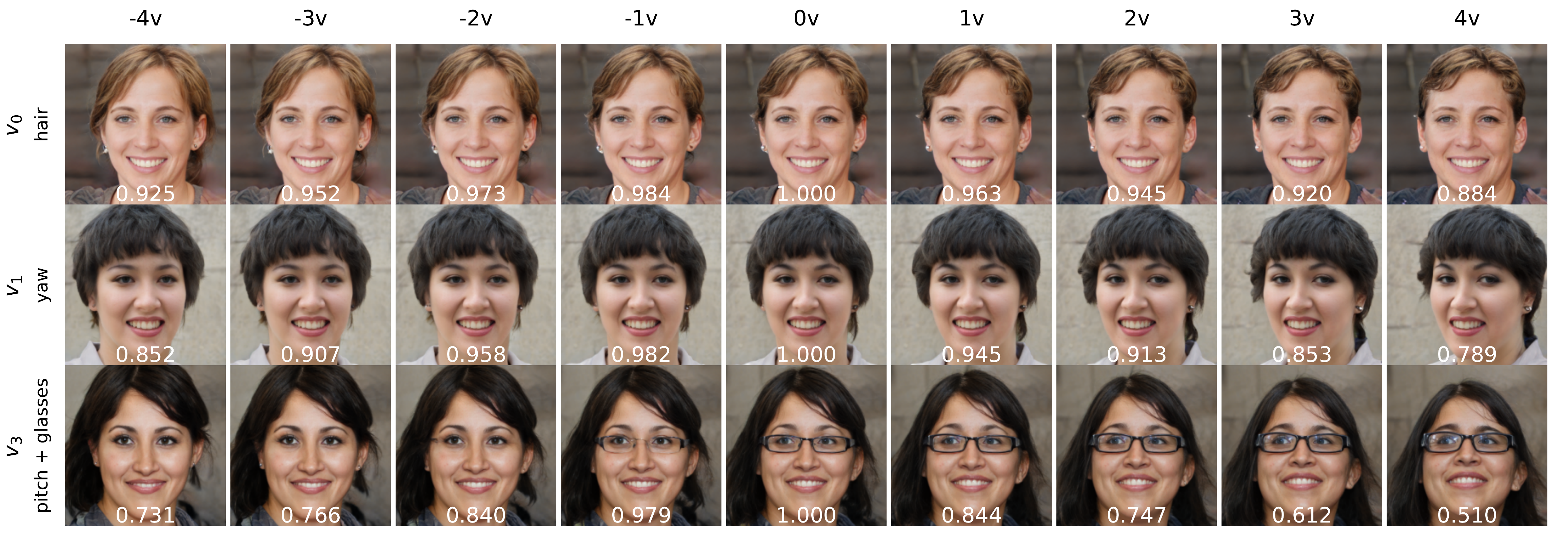}
\par\end{centering}
\caption{The top three principal components in the style space $\protect\W_{\protect\sty}^{+}$.\label{fig:pca_change_sty}}
\end{figure*}